\documentclass{amsart}
\usepackage{amsthm}
\usepackage{amsmath}
\usepackage{amsfonts}
\usepackage{amssymb}
\usepackage[usenames]{color}
\usepackage[mathscr]{euscript}
\usepackage[hidelinks]{hyperref} 
\usepackage[ruled,linesnumbered]{algorithm2e}
\usepackage[noend]{algpseudocode}
\def\MakeStep#1{\par\noindent\texttt{#1}}
\newcount\MyOutlineCount                
\def\MyOutline#1{\ifnum\MyOutlineCount>0 #1\fi}

\def\real{{\mathbb{R}}}

\def\optdec{_{{\normalfont\rm opt}}}
\def\findec{_{{\normalfont \rm end}}}
\def\errdec{_{{\normalfont \rm err}}}
\def\cmpdec{_{{\normalfont \rm cmp}}}
\def\pjdec{_{{\normalfont \rm pj}}}
\def\smdec{_{{\normalfont \rm sm}}}
\def\probdec{_{{\normalfont \rm prob}}}
\def\lebmeas#1.{{\mathcal L}^{\setbox0=\hbox{$#1\unskip$}\ifdim\wd0=0pt 1
    \else #1\fi}} 

\def\XXint#1#2#3{{\setbox0=\hbox{$#1{#2#3}{\int}$}
     \vcenter{\hbox{$#2#3$}}\kern-.5\wd0}}

\def\cvxs{{\mathcal{C}}}
\def\prcvxs{\pi_{\mathcal{C}}}
\def\sampdist{\mathcal{D}}
\def\filtration{\mathcal{F}}
\def\pcreatenrm#1#2{\expandafter\def\csname
  #1nrm\endcsname##1.{\left\|##1\right\|_{#2}} \expandafter\def\csname
  #1nrmname\endcsname{\left\|\,\cdot\,\right\|_{#2}}} 
\pcreatenrm{linf}{\infty}
\pcreatenrm{Clinf}{\cvxs,\infty}
\pcreatenrm{hilb}{}
\pcreatenrm{CHilb}{\cvxs}
\pcreatenrm{lip}{{\normalfont\rm Lip}}
\pcreatenrm{Clip}{\cvxs{\normalfont\rm Lip}}

\DeclareMathOperator\diam{diam}

\numberwithin{equation}{section} 
\theoremstyle{plain}
\newtheorem{lem}[equation]{Lemma}

\newtheorem{thm}[equation]{Theorem}

\theoremstyle{definition}
\newtheorem{defn}[equation]{Definition}

\theoremstyle{remark}

\newtheorem{rem}[equation]{Remark}
\begin{document}
\title[SGD-Optimality]{Optimality of the final model
  found via Stochastic Gradient Descent}
\author{Andrea Schioppa}
\address{Amsterdam, Noord Holland}
\email{ahisamuddatiirena+math@gmail.com}
\begin{abstract}
  We study convergence properties of Stochastic Gradient Descent (SGD)
  for convex objectives without assumptions on smoothness or strict
  convexity. We consider the question of establishing that with high
  probability the objective evaluated at the
  candidate minimizer returned by SGD is close to the minimal value of
  the objective. We compare this result concerning the final candidate
  minimzer (i.e.~the final model parameters learned after all gradient
  steps) to the online learning techniques of
  \cite{zinkevich-online-sgd} that take a rolling average of the model
  parameters at the different steps of SGD.
\end{abstract}
\maketitle
\tableofcontents
\section{Introduction}
\subsection{Motivation}
\par Stochastic Gradient Descent (SGD) is a popular approach to build
machine learning models by learning parameters that single out
an ``(approximately) optimal''  hypothesis in a given hypothesis
space. Main reasons for the popularity are the simplicity of the
algorithm and the ability to deal with real-life large
datasets. Moreover, SGD can be used also
to learn and optimize in real-time (e.g.~online learning) and the
gradient update rules can be refined (e.g.~using algorithms like
Adagrad, Adam or FTRL) to
improve convergence, especially in problems where the hypothesis space
is more complex either due to the large number of parameters
(e.g.~regressions with categorical features having high-cardinality)
or the complexity of the hypothesis (e.g.~neural networks).
\par In this note we only consider \emph{convex problems}, a case where
theoretical guarantees are well-understood \cite{bottou-sgd,
  nesterov-convex-programming, hazan-online-convex}.
Let us start with the classical mathematical setting of
minimizing a convex function $f:\cvxs\to\real$ where $\cvxs$ is a
convex compact subset of some Euclidean space $\real^N$. We want to
find a $u\in\cvxs$ that minimizes $f$ and for this Gradient Descent
(GD) uses steps in the direction of the subgradient $\partial f$ to
improve on an initial guess on the minimizer. A common hypothesis in this
case is that there is a uniform bound on the norm $\hilbnrm \partial
f.$ and a minor complication is to keep the constraint $u\in\cvxs$
during the minimization process.
\par In the typical supervised learning setting the situation is more
complex as the objective function $f$ is \emph{not directly known},
while one can sample objective functions
$f_t$ from a distribution $\sampdist$ with the guarantees $E[f_t]=f$
and $E[\partial f_t]=\partial f$\footnote{formally this involves
  saying that $E[\partial f_t]$ is \textbf{a} subgradient of $f$ when
  $f$ is not sufficiently smooth}. Concretely, one often has the case
that $f_t(u)=F(x_t,u)$ where $F$ is known but $x_t$ is sampled from a
distribution (e.g.~$x_t$ is the \emph{training example} consisting of
the predictive features and the target variable(s)). We will call the
variable $u$ parameters (e.g.~the parameters/weights of a linear
model). In this problem
there are two main complications:
\begin{enumerate}
\item gradient steps are in the direction of $\partial f$ only on
  average.
\item a probabilistic approach, e.g.~PAC-learning (see \cite{mohri-found},
  \cite[Ch.~9]{hazan-online-convex}) is needed to evaluate the goodness of the final
  hypothesis. In particular, the sequence of gradient updates
  generates a sequence of parameters $\{u_t\}_t$ which is no longer
  deterministic but a stochastic process.
\end{enumerate}
This work was motivated by the following question:
\begin{itemize}
\item[\textbf{Q1:}] Assume that we run SGD for $T$ iterations; how good of an approximate
  minimizer of $f$ is the  \emph{final parameter} $u_{T+1}$ returned by SGD?
  More precisely, can we
claim that if $T$ is sufficiently large then $f(u_{T+1})$ is close to the
minimal value of $f$ with high probability?
\end{itemize}
Despite the amount of literature on SGD we were not able to find an
answer to \textbf{Q1} that we found satisfactory\footnote{Overlooks
  here are to blame on me!}. We are aware of results  on the
expectation either of $E[f(u_{T+1})]$ or $E[u_{T+1}]$
(for example the recent~\cite{nguyen-sgd}) or results about $u_{T+1}$
under additional assumptions on the stochastic process generated by
SGD (for example \cite{zinkevich-parallelized}). In particular, in
\cite{zinkevich-parallelized} the authors are in a
sufficiently smooth and regularized setting so that the gradient updates
result in a contraction in the parameter space. Using Wasserstein
distances, they can then guarantee probabilitstic results on $u_{T+1}$.
\par To study \textbf{Q1} we consider two strategies. The first one
uses the connection between online learning and convex optimization of
\cite{zinkevich-online-sgd} and replaces $u_{T+1}$ by a running average of the parameter
weights. Our understanding is that the averaging reduces the
uncertainties in the gradient steps. However, this approach answers
 \textbf{Q1} \emph{only partially}. In the second approach we work directly with
$u_{T+1}$ but we need to overcome 3 technical issues:
\begin{itemize}
\item[\textbf{I1:}] the subgradients $\partial f_t$ are not uniformly
  Lipschitz, this breaks some arguments in gradient descent.
\item[\textbf{I2:}] martingales arguments in the PAC-framework do not work well when
  the projection onto the convex set $\cvxs$ is not linear.
\item[\textbf{I3:}] we need a slightly improvement of Hoeffding's
  concentration inequality~\cite{hoeffding-concentration}
  that uses one of Doob's maximal inequalities.
\end{itemize}
\subsection{Warm-up: the deterministic case}
\par As a warm-up case let us consider the case in which the sampling
distribution $\sampdist$ is concentrated at the objective $f$. This is
the case where the objective $f$ is known and available (\cite{nesterov-convex-programming}
for overview of results on convex optimization). We assume that the
set $\cvxs\subset\real^N$ is compact and convex and let $\pi_\cvxs:\real^N\to\cvxs$
denote the projection onto $\cvxs$ which is known to be
$1$-Lipschitz. Let $u\optdec\in\cvxs$ denote a point where $f$ attains
the minimum in $\cvxs$. The first strategy is to choose a sequence of learning
rates $\{\varepsilon_t\}_{t=1}^T$, take gradient steps, project back
to $\cvxs$ and then take the average of the parameters at different
steps (which lies in $\cvxs$ by convexity). We call this algorithm
\emph{Running Average Projected Gradient Descent}, abb.~RAPGD and
pseudocode~\ref{al:RAPGD}. The running average is because one can
do the computation of the final average by keeping running sums
($v,\rho$ in the pseudocode) over
the parameters and the learning rates.
\begin{algorithm}
  \SetKwInOut{Input}{input}\SetKwInOut{Output}{output}
   \Input{convex objective $f$, compact convex set $\cvxs$ with
     projection $\pi_\cvxs$, sequence of learning rates
     $\{\varepsilon_t\}_{t=1}^T$,
     initial point $u_1\in\cvxs$.}     
   \Output{approximate minimizer $u\findec$.}
   \Begin{
     $\rho\gets 0$\;
     $v\gets 0\in\real^N$\;
     \For{$t \gets 1$ \KwTo $T$}{ 
      $u_{t+1/2}\gets u_t -\varepsilon_t\partial f(u_t)$\;
      $u_{t+1}\gets\prcvxs(u_{t+1/2})$\;
      $\rho\gets\rho+\varepsilon_t$\;
      $v\gets v+\varepsilon_t u_t$\;
      }
      $u\findec\gets \frac{v}{\rho}$
    }
    \caption{Running Average Projected Gradient Descent}\label{al:RAPGD}
  \end{algorithm}
  RAPGD is analyzed in Theorem~\ref{thm:rapgd} using the analysis of
  \cite{zinkevich-online-sgd}. In particular, under the assumption that the norm of
  $\partial f$ is uniformly bounded on $\cvxs$ one shows that:
  \begin{equation}
    f(u\findec)-f(u\optdec)=O\left(\frac{\sum_{t=1}^T{\varepsilon_t^2}}
    {\sum_{t=1}^T{\varepsilon_t}}\right);
  \end{equation}
  taking for example $\varepsilon_t=\frac{1}{\sqrt{t}}$ one gets a
  bound:
  \begin{equation}
    f(u\findec)-f(u\optdec)=O\left(\frac{\log T}{\sqrt{T}}\right).
  \end{equation}
  \par The second strategy is to take gradient steps followed by
  projection. We call this algorithm \emph{Plain Projected Gradient
    Descent}, abb.~PPGD, pseudocode~\ref{al:ppgd}. 
  \begin{algorithm}
  \SetKwInOut{Input}{input}\SetKwInOut{Output}{output}
   \Input{convex objective $f$, compact convex set $\cvxs$ with
     projection $\pi_\cvxs$, sequence of learning rates
     $\{\varepsilon_t\}_{t=1}^T$,
     initial point $u_1\in\cvxs$.}     
   \Output{approximate minimizer $u\findec$.}
   \Begin{
     \For{$t\gets 1$ \KwTo $T$}{
       $u_{t+1/2}\gets u_t -\varepsilon_t\partial f(u_t)$\;
       $u_{t+1}\gets\prcvxs(u_{t+1/2})$\;
     }
     $u\findec\gets u_{T+1}$
     }
    \caption{Plain Projected Gradient Descent}\label{al:ppgd}
  \end{algorithm}
  The analysis of this approach is done in Theorem~\ref{thm:ppgd_main}
  and Remark~\ref{rem:ppgd_main}. Under the assumption of a uniform
  bound $L$ on the Lipschitz constant of $\partial f$ we obtain a
  bound:
  \begin{equation}
    f(u\findec)-f(u\optdec)=O\left(\frac{L}{\sqrt{T}}\right);
  \end{equation}
  as remarked in Remark~\ref{rem:ppgd_main} a more careful analysis
  following \cite[Corollary~2.1.2]{nesterov-convex-programming} would give
  \begin{equation}
    f(u\findec)-f(u\optdec)=O\left(\frac{L}{{T}}\right).
  \end{equation}
  Here the role of the Lipschitz hypothesis is to guarantee that the
  objective goes down after every iteration, not just on average after
  some iterations. Moreover, the nonlinearity of the projection
  $\pi_\cvxs$ introduces some complications. In particular, we
  introduce a notion of local norm (Definition~\ref{defn:local_norm})
  for vectors with respect to the convex set $\cvxs$; in this setting,
  even though $\partial f$ does not vanish at a minimizer
  $u\optdec\in\cvxs$, one is guaranteed that $\CHilbnrm\partial
  f(u\optdec).=0$.
  \subsection{The stochastic setting} In the general case we do not
  have direct access to $f$ but we base our gradient steps on sampling
  from a distribution $\sampdist$ of convex functions. In \emph{Running
    Average Projected Stochastic Gradient Descent}, abb.~RAPSGD and
  pseudocode~\ref{al:raspgd} we proceed similarly to RAPGD. At each
  step we take a sample from $\sampdist$ independent of what happens
  at the previous steps, update the gradient and project back to $\cvxs$.
 \begin{algorithm}
  \SetKwInOut{Input}{input}\SetKwInOut{Output}{output}
   \Input{convex objective $f$, compact convex set $\cvxs$ with
     projection $\pi_\cvxs$, sequence of learning rates
     $\{\varepsilon_t\}_{t=1}^T$,
     initial point $u_1\in\cvxs$, distribution of convex functions
     $\sampdist$ whose average is $f$.}
   \Output{approximate minimizer $u\findec$.}
   \tcc{Note that the sequence of iterations
   gives rise to a filtration $\{\filtration_t\}_t$.}     
   \Begin{
     $\rho\gets 0$\;
     $v\gets 0\in\real^N$\;
     \For{$t \gets 1$ \KwTo $T$}{ 
      Sample $f_{t+1/2}$ from $\sampdist$ independently of $\filtration_t$.\;
      $u_{t+1/2}\gets u_t -\varepsilon_t\partial f_{t+1/2}(u_t)$\;
      $u_{t+1}\gets\prcvxs(u_{t+1/2})$\;
      $\rho\gets\rho+\varepsilon_t$\;
      $v\gets v+\varepsilon_tu_t$\;
      }
      $u\findec\gets \frac{v}{\rho}$
  }
   \caption{Running Average Stochastic Projected Gradient Descent}
   \label{al:raspgd}
 \end{algorithm}
 The theoretical guarantee that one proves in Theorem~\ref{thm:raspgd}
 is that, for an appropriate choice of the learning rates,
 with a number of steps $T=O(\varepsilon^{-2})$ with
 probability $1-O(\varepsilon^2)$ one has:
 \begin{equation}
   f(u\findec)-f(u\optdec)\lessapprox 2\varepsilon\log\frac{1}{\varepsilon}.
 \end{equation}
 Note that the notation $\lessapprox$ is an asymptotic notation. For
 an error parameter $\varepsilon$ writing $a(\varepsilon)\lessapprox
 b(\varepsilon)$ means that
 $\lim_{\varepsilon\searrow0}\frac{a(\varepsilon)}{b(\varepsilon)}\le
 1$, while for the number of steps parameters $T$ writing
 $a(T)\lessapprox b(T)$ means that
 $\lim_{T\nearrow\infty}\frac{a(T)}{b(T)}\le 1$. The proof of
 Theorem~\ref{thm:raspgd} relies on that of Theorem~\ref{thm:rapgd}
 and the probability bound is obtained via Hoeffding's inequality. As
 we obtain $u\findec$ as an average, the nonlinearity of $\pi_\cvxs$ is
 dealt with by using Jensen's inequality.
 \par The second approach allows to get a theoretical guarantee for
 the final parameters returned by the algorithm. We call this
 algorithm \emph{Smoothed Stochastic Gradient Descent}, abb.~SSGD and
 pseudocode~\ref{al:ssgd}. 
 \begin{algorithm}
  \SetKwInOut{Input}{input}\SetKwInOut{Output}{output}
   \Input{convex objective $f$, compact convex set $\cvxs$ with
     a penalization function $\psi_{\cvxs}$ and the projection
     $\pi_\cvxs$,
     sequence of learning rates
     $\{\varepsilon_t\}_{t=1}^T$,
     initial point $u_1\in\cvxs$, distribution of convex functions
     $\sampdist$ whose average is $f$ and whose subgradients are bounded
     in norm by $G$ (i.e.~for $\hat{f}\in\sampdist$ on has
     $\hilbnrm\partial \hat{f}.\le G$), the uniform distribution
     $\eta_{\varepsilon\smdec}$ on the ball of radius $\varepsilon\smdec$.}
   \Output{approximate minimizer $u\findec$.}
   \tcc{Note that the sequence of iterations
     gives rise to a filtration $\{\filtration_t\}_t$.}
   \Begin{
 \For{$t\gets 1$ \KwTo $T$}{
   $f_{t+1/2} \gets \sampdist$ independently of $\filtration_t$\;
   $v\gets\eta_{\varepsilon\smdec}$ independently of $\filtration_{t+1/2}$\;
   $u_{t+1}\gets u_t -\varepsilon_t[\partial f_{t+1/2}(u_t-v)
    +2G\partial\psi_{\cvxs}(u_t-v)]$\;
}
    $u\findec\gets \pi_{\cvxs}({u_{T+1}})$
}
\caption{Smoothed Stochastic Gradient Descent}\label{al:ssgd}
\end{algorithm}
SSGD differs from RASPGD in two respects:
\begin{enumerate}
\item the constraint $u_t\in\cvxs$ is not strongly enforced at
  each step, but only at the final step taking a projection (line
  7). At the general step one uses a penalization function
  $\psi_\cvxs$ introduced in Lemma~\ref{lem:gauge}. This idea is not
  too different from that of relaxing a constraint by adding a
  penalization to the objective function via 
  Lagrange multipliers.
\item a perturbation term (see \cite[Sec.~2.3.2, Algorithm~4]{hazan-online-convex}) is sampled (line 4) to smooth
  out the gradient updates (line 5).
\end{enumerate}
\par In Theorem~\ref{thm:anassgd} we prove that, for a particular choice
of the learning rates, with a number of steps $T=O(\varepsilon^{-6})$
with probability $1-O(\varepsilon)$ one has that:
\begin{equation}
  f(u\findec)-f(u\optdec)\lessapprox 128G^2\varepsilon\log\frac{1}{\varepsilon}.
\end{equation}
On the other hand, if $\partial f$ was known to be Lipschitz, one
would require a number of steps $T=O(\varepsilon^{-4})$. The ability
to bound the properties of the final parameter $u_{T+1}$ (note that in
Theorem~\ref{thm:anassgd} we show that the distance bewteen $u_{T+1}$
and $u\findec=\pi_{\cvxs}(u_{T+1})$ is $O(\varepsilon)$ so one can use
$u_{T+1}$ and $u\findec$ interchangeably even though only the latter
is guaranteed to lie in $\cvxs$) come at the
cost of a slower convergence (though we conjecture that with an
adaptative selection of the learning rate depending on the size of the
current gradient one might improve to
$O(\varepsilon^{-2}\log\frac{1}{\varepsilon})$).  SSGD handles the issues
we mentioned above as follows:
\begin{itemize}
\item[\textbf{I1:}] the smoothing (lines 4 and 5) slightly perturbs
  the objective $f$ to one which has Lipschitz subgradient. This idea
  is essentially a use of mollifications in real analysis.
\item[\textbf{I2:}] the penalization $\psi_\cvxs$ allows to avoid
  taking projections; thus we can use linearity in taking
  expectations.
\item[\textbf{I3:}] in deriving~\eqref{eq:anassgd10} we use a stronger
  version of Hoeffding's concentration inequality; as observed by
  Hoeffding in \cite[equation~2.17]{hoeffding-concentration}
  this comes almost for free combining his proof
  with one of Doob's maximal inequalities for martingales.
\end{itemize}
\subsection{Summary} We address the question concerning wether the
final parameters $u_{T+1}$ returned by SGD is minimizing the convex
objective $f$ up to a small error:
\begin{itemize}
\item In Theorem~\ref{thm:anassgd} we show that $f(u_{T+1})$ is within
  distance $O(\varepsilon\log1/\varepsilon)$ from the minimum with
  probability $1-O(\varepsilon)$ if the number of gradient descent
  steps $T$ is $O(\varepsilon^{-6})$.
\item In Theorem~\ref{thm:raspgd} we show that the same holds if
  instead of $u_{T+1}$ one considers a rolling average of the
  parameters with $T$ now $O(\varepsilon^{-2})$.
\end{itemize}
The paper is organized in two section. In the first one we deal with
the deterministic case in which each gradient step works directly with
$f$. This is mainly for illustrative purposes. In the second section
we deal with the case in which at each step we sample from a
distribution of convex functions whose mean is $f$.
\section{The determistic case}
\par In this section we analyze RAPGD and PPGD. In order to analyze
PPGD we need to introduce a notion of local norm and prove a
geometric result, Theorem~\ref{thm:localgeoprop}.
\subsection{Analysis of RAPGD} Following \cite{zinkevich-online-sgd} we prove:
\begin{thm}[Analysis of RAPGD]
  \label{thm:rapgd}
  If $u\optdec\in\cvxs$ is a minimizer of $f$ in $\cvxs$ and if
  $\sup_{u\in\cvxs}\|\partial f(u)\|<\infty$,
  letting
  \begin{equation}
    \Clinfnrm\partial f.=\sup_{u\in\cvxs}\|\partial f(u)\|,
  \end{equation}
  then
  \begin{equation}
    f(u\findec)-f(u\optdec)\le\frac{\hilbnrm u_1-u\optdec.^2+
      \Clinfnrm\partial f.^2
      \sum_{t=1}^T\varepsilon_t^2}{2\sum_{t=1}^T\varepsilon_t}.
  \end{equation}
\end{thm}
\begin{proof}
  Using convexity of $f$ and the definition of the subgradient we
  obtain
  \begin{equation}
    \label{eq:rapgd0}
    \begin{split}
    \varepsilon_t (f(u_t)-f(u\optdec)) &\le\varepsilon_t\langle
    \partial f(u_t), u_t-u\optdec\rangle \\
    &=\langle u_t-u_{t+1/2},u_t-u\optdec\rangle\\
    &=\langle (u_t-u\optdec)-(u_{t+1/2}-u\optdec), u_t-u\optdec\rangle.
  \end{split}
\end{equation}
Using that projection onto $\cvxs$ is $1$-Lipschitz and expanding the
Hilbert norm into products we get:
\begin{equation}
  \label{eq:rapgd1}
  \begin{split}
    \hilbnrm u_{t+1} - u\optdec.^2 &\le\hilbnrm u_{t+1/2}-u\optdec.^2
    =\hilbnrm u_{t+1/2}-u_t+u_t-u\optdec.^2\\
    &=\varepsilon_t^2\hilbnrm\partial f(u_t).^2 +
    \hilbnrm u_t-u\optdec.^2 + 2\langle u_{t+1/2}-u_t, u_t-u\optdec\rangle.
  \end{split}
\end{equation}
Substituting~(\ref{eq:rapgd1}) into~(\ref{eq:rapgd0}) we get:
\begin{equation}
  \label{eq:rapgd2}
  \varepsilon_t (f(u_t)-f(u\optdec)) \le
  \frac{\hilbnrm u_t-u\optdec.^2 - \hilbnrm u_{t+1}-u\optdec.^2
    +\varepsilon_t^2\hilbnrm\partial f(u_t).^2}{2};
\end{equation}
summing in $t$ we get:
\begin{equation}
  \label{eq:rapgd3}
  \sum_{t=1}^T   \varepsilon_t (f(u_t)-f(u\optdec)) \le
  \frac{\hilbnrm u_1-u\optdec.^2+\sum_{t=1}^T\varepsilon_t^2
          \Clinfnrm\partial f.^2}{2},
      \end{equation}
      and application of Jensen's inequality finally yields
      \begin{equation}
        \label{eq:rapgd4}
        \sum_{t=1}^T\varepsilon_t (f(u\findec)-f(u\optdec))
        \le  \frac{\hilbnrm u_1-u\optdec.^2+\sum_{t=1}^T\varepsilon_t^2
          \Clinfnrm\partial f.^2}{2}.
      \end{equation}
    \end{proof}
\subsection{Preliminary analysis of PPGD} A preliminary analysis of
PPGD is based on the following Lemma which analyzes the effect of a
single gradient step. Note that the term $-\varepsilon_t L$
in~\eqref{eq:ppgdlem0} might be improved to $-\varepsilon L/2$ (compare \cite[Lemma~1.2.3]{nesterov-convex-programming}).
\begin{lem}\label{lem:ppgdlem}
  Let $\partial f$ be $L$-Lipschitz; then in plain projected
  gradient descent one has
  \begin{equation}\label{eq:ppgdlem0}
    f(u_{t+1})-f(u_t) \le -(1-\varepsilon_t L)\langle\partial f(u_t),u_t-u_{t+1}\rangle.
  \end{equation}
  Moreover,
  \begin{equation}\label{eq:ppgdlem0bis}
    \langle\partial f(u_t),u_t-u_{t+1}\rangle\ge0,
  \end{equation}
  and thus if $\varepsilon_t\le\frac{1}{L}$ then
  $\{f(u_t)\}_{t=1}^{T+1}$ is non-increasing.
\end{lem}
\begin{proof}
  Applying first convexity of $f$ and the definition of the
  subgradient $\partial f$ and the Lipschitz condition on $\partial f$
  we obtain:
  \begin{equation}    \label{eq:ppgdlem1}
    \begin{split}
      f(u_{t+1})-f(u_t) & \le\langle \partial f(u_{t+1}),
      u_{t+1}-u_t\rangle \\
      &=\langle \partial f(u_{t+1})-\partial f(u_t), u_{t+1}-u_t
      \rangle + \langle\partial f (u_t),u_{t+1}-u_t\rangle \\
      &\le L\hilbnrm u_{t+1}-u_t.^2+\langle\partial f
      (u_t),u_{t+1}-u_t\rangle.      
    \end{split}
  \end{equation}
  Rewriting $\hilbnrm u_{t+1}-u_t.^2$ as
  \begin{equation}
    \hilbnrm u_{t+1}-u_t.^2 = \langle
    u_{t+1}-u_t,u_{t+1}-u_{t+1/2}\rangle
    + \langle u_{t+1}-u_t,u_{t+1/2}-u_t\rangle,
  \end{equation}
  and recalling that  $\langle
  u_{t+1}-u_t,u_{t+1}-u_{t+1/2}\rangle\le0$ as
  $u_{t+1}$ is the point of $\cvxs$ closest to $u_{t+1/2}$ and $u_t\in\cvxs$,
  we obtain
  \begin{equation}
    \label{eq:ppgdlem2}
    \hilbnrm u_{t+1}-u_t.^2\le\langle
    u_{t+1}-u_t,-\varepsilon_t\partial f(u_t)\rangle;
  \end{equation}
  substitution into \eqref{eq:ppgdlem1} yields \eqref{eq:ppgdlem0}.
  Finally, for \eqref{eq:ppgdlem0bis} we assume that $\partial
  f(u_t)\ne0$; then from the decomposition
  \begin{equation}
    u_{t+1} = u_t + R\partial f(u_t) + v^\perp,
  \end{equation}
  where $v^\perp$ is orthogonal to $\partial f(u_t)$, we
  observe that if \eqref{eq:ppgdlem0bis} did not held, $R$ would be
  less than $0$ and hence $u_t$ closer to $u_{t+1/2}$ than $u_{t+1}$,
  contradicting that $u_{t+1}$ is the point of $\cvxs$ closest to $u_{t+1/2}$.
\end{proof}
\subsection{Geometric results} As we use the non-linear projection
$\pi_\cvxs$ onto $\cvxs$ we introduction a notion of local norm to
measure, having fixed a vector $u$ and a $p\in\cvxs$, how big is the
effective norm of $u$ when we consider only steps that start from $p$
and hit the sphere centered at $p$ of radius $r$ in some point of
$\cvxs$. We then prove a geometric result that we need when we want to
relate gradient steps to the local norm. As in the stochastic case we
use a penalization function $\psi_\cvxs$ to weakly enforce the
constraint $u_t\in\cvxs$, we might have skipped a discussion of local
norms, but we think it is useful to get an idea of the complications
that can arise because of the nonlinearity of projections.
\begin{defn}[Local norm]\label{defn:local_norm}
  Given $p\in\cvxs$, a vector $u$ and
  $r>0$, we define the \textbf{local norm of $u$ at $p$ at
    scale $r$ relatively to $\cvxs$} as:
  \begin{equation}
    \CHilbnrm u.(p;r) = \begin{cases}
      
      \sup_{v\in\cvxs, \hilbnrm v-p.= r}\max\left(\frac{\langle u,
          v-p\rangle}{r},0\right)
      &\text{if there is a $v\in\cvxs$: $\hilbnrm v-p.=r$} \\
      0&\text{otherwise}.
  \end{cases}
\end{equation}
\end{defn}
\begin{lem}
  The map $r\mapsto\CHilbnrm u.(p;r)$ is non-increasing and
  $\limsup_{r\searrow0}\CHilbnrm u.(p;r)\le\hilbnrm u.$.
\end{lem}
\begin{proof}
  Fix $\varepsilon>0$ and $r_1\ge r_0>0$. Choose $v_1\in\cvxs$ with
  $\hilbnrm v_1-p.=r_1$ such that:
  \begin{equation}
    \max\left(
      \frac{\langle u, v_1-p\rangle}{r_1}
      , 0
      \right)\ge\CHilbnrm u.(p;r_1)-\varepsilon;
    \end{equation}
    writing $v_1=p+w_1$ we have $\hilbnrm w_1.=r_1$; as $p\in\cvxs$
    and $\cvxs$ is convex, the point $v_0 = p + r_0w_1/r_1$ lies also
    in $\cvxs$ and we have:
    \begin{equation}
       \max\left(
      \frac{\langle u, v_0-p\rangle}{r_0}
      , 0
      \right)=       \max\left(
      \frac{\langle u, v_1-p\rangle}{r_1}
      , 0
      \right);
    \end{equation}
    we thus conclude that
    \begin{equation}
      \CHilbnrm u.(p;r_0)\ge \CHilbnrm u.(p;r_1).
    \end{equation}
    Note however, that for any $r>0$ we have
    \begin{equation}
       \max\left(
      \frac{\langle u, v-p\rangle}{r}
      , 0
      \right)\le\hilbnrm u.
    \end{equation}
    whenever $\hilbnrm v-p.=r$; thus $\limsup_{r\searrow0}\CHilbnrm u.(p;r)\le\hilbnrm u.$.
\end{proof}
\begin{thm}\label{thm:localgeoprop}
  Whenever $u$ is a unit-norm vector, i.e.~$\hilbnrm u.=1$,
  for the local norm we have the fundamental inequality linking it to
  projections, where the implied constants are universal (also in the
  underlying $\real^N$'s dimension):
  \begin{equation}\label{geom:ineq}
      \langle u, \prcvxs(p+ru)-p\rangle
     \ge \frac{r}{2}\CHilbnrm u.^2(p;r).
  \end{equation}
\end{thm}
\begin{proof}
  \MakeStep{Step 1: a weak bound.}
  Without loss of generality we can assume that $p$ is the origin $0$.
  Note that if $\prcvxs(ru)\ne0$ then the left hand side of
  \eqref{geom:ineq} would have to be $\ge 0$, otherwise $0\in\cvxs$
  would be closer to $ru$ than $\prcvxs(ru)$.
  We conclude that the left hand side of
  \eqref{geom:ineq} is always
  nonnegative.
  
  If there were no $v\cmpdec\in\cvxs$ such that $\hilbnrm v\cmpdec.=r$
  the right side of  \eqref{geom:ineq} would be $0$. If
  $\prcvxs(ru)=0$ then, as the angle between $ru$ and $v\cmpdec$ would
  have to be $\ge\pi/2$, the right hand side of  \eqref{geom:ineq}
  would be $0$ too.

  We thus focus on the case in which $v\pjdec=\prcvxs(ru)\ne0$, and
  there is a $v\cmpdec\in\cvxs$ such that $\hilbnrm v\cmpdec.=r$. Let
  us define the angle $\alpha=\measuredangle(u,0,v\pjdec)$; we know
  that $\alpha\in[0,\pi/2]$ from the above discussion. We now have a
  weak bound:
  \begin{equation}
    \label{eq:geom_ineq_weak}
    \measuredangle(u,0,v\cmpdec)\ge\alpha,
  \end{equation}
  otherwise the point $\frac{\hilbnrm v\pjdec.}{\hilbnrm v\cmpdec.}v\cmpdec$,
  which belongs to $\cvxs$ by convexity, would be closer to $ru$ than
  $v\pjdec$.

  \MakeStep{Step 2: planar reduction}.
  Let
  \begin{equation}
    v\cmpdec = v\pjdec + w_P + w_P^\perp
  \end{equation}
be an orthogonal decomposition of $v\cmpdec-v\pjdec$ with respect to
the plane $P$ spanned by $u$ and $v\pjdec$. By the properties of
projections onto convex sets:
\begin{equation}
  \langle ru-v\pjdec, v\cmpdec-v\pjdec\rangle\le0;
\end{equation}
thus $\langle ru-v\pjdec,w_P\rangle\le0$. If we let $\tilde
v\cmpdec=v\pjdec+w_p$ we have that:
\begin{equation}\label{geom:ineq_var}
  \langle u,v\cmpdec\rangle = \langle u,\tilde v\cmpdec\rangle.
\end{equation}
To establish \eqref{geom:ineq} we can thus replace $v\cmpdec$ by
$\tilde v\cmpdec$ even though, in general, $\tilde v\cmpdec$ does not
belong to $\cvxs$. By expanding $\hilbnrm ru-\tilde v\cmpdec.^2$ we
find:
\begin{equation}\label{geom:two_red}
  \begin{split}
    \langle ru,\tilde{v\cmpdec}\rangle &=
    \frac{r^2+\hilbnrm\tilde v\cmpdec.^2-
      \hilbnrm ru-\tilde v\cmpdec.^2}
    {2} \\
    &\stackrel{\le}{\scriptstyle\eqref{geom:ineq_var}}\frac{
      r^2+\hilbnrm v\cmpdec.^2-\hilbnrm ru-v\pjdec.^2
      -\hilbnrm w_P.^2}{2} \\
    &=\frac{r^2+\hilbnrm v\pjdec.^2-\hilbnrm ru-v\pjdec.^2}{2}
    +\langle w_P,v\pjdec\rangle;
  \end{split}
\end{equation}
as also $\measuredangle (u,v\pjdec,0)\ge\pi/2$ the $w_P$ that
maximizes the right hand side of \eqref{geom:two_red} would have to be
orthogonal to $ru-v\pjdec$. Let $\hat v\cmpdec$ be obtained from
$\tilde v\cmpdec$ by replacing $w_P$ with the vector of the same norm
and orthogonal to $ru-v\pjdec$ so to maximize the
right hand side of \eqref{geom:two_red}.
\MakeStep{Step 3: trigonometric inequalities.}
Let $\beta=\measuredangle (ru,v\pjdec,0)$ so that
$\beta\in[\alpha,\pi/2]$;
applying the law of sines to the triangle $\triangle(ru,v\pjdec,0)$
we find:
\begin{equation}
  \hilbnrm v\pjdec.=\frac{r\sin(\alpha+\beta)}{\sin\beta}.
\end{equation}
Letting $\delta=\measuredangle(v\pjdec,\hat v\cmpdec,0)$ and
applying the law of sines to $\triangle(v\pjdec,\hat v\cmpdec,0)$ one
has:
\begin{equation}
  \sin\delta = \frac{\hilbnrm v\pjdec.}{\hilbnrm\hat v\cmpdec.}
  \sin(3\pi/2-\beta).
\end{equation}
Thus we obtain
\begin{equation}
  \label{geom:final_ineq}
  \frac{\langle v\pjdec,u\rangle}
  {\langle \hat v\cmpdec,u\rangle} = \frac{r}{\hilbnrm\hat v\cmpdec.}
  \frac{\sin(\alpha+\beta)}{\sin\beta}
  \frac{\cos\alpha}{\cos(\alpha+\beta-\pi/2-\delta)}.
\end{equation}
We now claim that the right hand side of \eqref{geom:final_ineq} is at
leat $\frac{\cos\alpha}{2}$. Indeed setting that right hand side to be
$\ge \frac{\cos\alpha}{2}$ we get:
\begin{equation}\label{geom:final_ineq2}
  \begin{split}
    2\frac{r}{\hilbnrm\hat v\cmpdec.}
    \frac{\sin(\alpha+\beta)}{\sin\beta}&\ge
    \sin\beta(\cos(\alpha+\beta-\pi/2)\cos\delta +
    \sin(\alpha+\beta-\pi/2)\sin\delta)\\
    &=\sin\beta\cos\delta\sin(\alpha+\beta)-
    \cos(\alpha+\beta)\sin(\alpha+\beta)\frac{r}{\hilbnrm\hat
      v\cmpdec.}
    \sin(3\pi/2-\beta);
  \end{split}
\end{equation}
in this form we see that \eqref{geom:final_ineq2} holds.
We conclude observing that:
\begin{equation}
  \begin{split}
    \frac{\langle v\pjdec,u\rangle}{r} &\ge
    \frac{\cos\alpha}{2}\frac{\langle\hat v\cmpdec,u\rangle}{r} \\
    &\stackrel{\ge}{\scriptstyle\texttt{Step 2}}
      \frac{\cos\alpha}{2}\frac{\langle v\cmpdec,u\rangle}{r} \\
      &\stackrel{\ge}{\scriptstyle\texttt{Step 1}}
      \frac{1}{2}\left(\frac{\langle v\cmpdec,u\rangle}{r}\right)^2.
  \end{split}
\end{equation}
\end{proof}
\subsection{Analysis of PPGD} We can now prove
Theorem~\ref{thm:ppgd_main}. In this case the best choice of learning
rates is constant in $t$, see Remark~\ref{rem:ppgd_main}.
\begin{thm}\label{thm:ppgd_main}
  If $u\optdec\in\cvxs$ is a minimizer of $f$ in $\cvxs$,
  if $\partial f$ is $L$-Lipschitz
  and if
  for each $t\in\{1,\cdots,T\}$ one has:
  \begin{equation}
    1-\varepsilon_t L\ge\gamma>0;
  \end{equation}
  then one either has:
  \begin{equation}\label{eq:ppgd_main0}
    f(u\findec)-f(u\optdec)\le\varepsilon\errdec,
  \end{equation}
  or for some $t^*\in\{1,\cdots,T\}$ one has:
  \begin{equation}\label{eq:ppgd_main1}
    \CHilbnrm\partial f(u_{t^*}).(u_{t^*};\varepsilon_{t^*}
    \hilbnrm\partial f(u_{t^*}).) \le\left(
      2\frac{f(u_1)-f(u\optdec)-\varepsilon\errdec}{\gamma
        \sum_{t=1}^T\varepsilon_t}
      \right)^{1/2}.
  \end{equation}
\end{thm}
\begin{proof}
  From Lemma~\ref{lem:ppgdlem} we get:
  \begin{equation}
    \label{eq:ppgdthm0}
    f(u\findec)-f(u_1)\le-\gamma\sum_{t=1}^T\langle
    \partial f(u_t), u_t-u_{t+1}\rangle;
  \end{equation}
  we claim that~\eqref{eq:ppgdthm0} implies:
  \begin{equation}
    \label{eq:ppgdthm1}
    f(u\findec)-f(u_1)\le
    -\frac{\gamma}{2}\sum_{t=1}^T\varepsilon_t
    \CHilbnrm\partial f(u_t).^2(u_t;\varepsilon_t\hilbnrm\partial f(u_t).);
  \end{equation}
  indeed, if some $\partial f(u_t)$ is $0$ there is nothing to prove,
  otherwise we apply Theorem~\ref{thm:localgeoprop} to the unit vector
  $\partial f(u_t)/\hilbnrm \partial f(u_t).$.
  Adding $f(u_1)-f(u\optdec)$ to~\eqref{eq:ppgdthm1} we obtain:
  \begin{equation}
    f(u\findec)-f(u\optdec)\le
    f(u_1)-f(u\optdec)-\frac{\gamma}{2}\sum_{t=1}^T
    \varepsilon_t     \CHilbnrm\partial f(u_t).^2(u_t;\varepsilon_t\hilbnrm\partial f(u_t).);
  \end{equation}
  thus if~\eqref{eq:ppgd_main0} is violated, for some $t^*$
  (esplicitly a $t$ for which $\CHilbnrm\partial
  f(u_t).^2(u_t;\varepsilon_t\hilbnrm\partial f(u_t).)$ is minimal),
  we have that~\eqref{eq:ppgd_main1} must hold.
\end{proof}
\begin{rem}\label{rem:ppgd_main}
  Note that Theorem~\ref{thm:ppgd_main} gives, choosing
  $\varepsilon_t=1/(2L)$ for each $t$ an order of iterations of
  $O(L/\varepsilon\errdec^2)$ to achieve an error $\varepsilon\errdec$
  in the minimization condition~\eqref{eq:ppgdthm0}. This is not
  optimal however, as for $f$ $L$-Lipschitz one can achieve a better
  bound $O(L/\varepsilon)$ as shown in
  \cite[Corollary~2.1.2]{nesterov-convex-programming}.
  However, we were not able to adapt that argument to the next
stochastic case.
\end{rem}
\section{The stochastic case}
\par In this section we analyze RASPGD and SSGD. We first deal with
SSGD, whose proof is more involved.
\subsection{An extension result} The following result is added for
completeness. It shows that if $f$ is just defined on $\cvxs$ we can
extend it to all of $\real^N$ while keeping the gradients
bounded. This extension property is needed in the smoothing step.
\begin{lem}\label{lem:ext_convex}
  Let $\cvxs\subset\real^N$ a convex compact subset and
  $f:\cvxs\to\real$ be a continuous convex function such that there is a choice
  $\partial f$ of the subradient of $f$ such that:
  \begin{equation}
    \label{eq:ext_convex1}
    \sup_{x\in\cvxs}\hilbnrm \partial f(x).\le G<\infty.
  \end{equation}
  Then there is a convex extension $\tilde f:\real^N\to\real$ of $f$
  such that there is a choice of the subgradient $\partial \tilde f$
  such that:
  \begin{equation}
    \label{eq:ext_convex2}
        \sup_{x\in\real^N}\hilbnrm \partial \tilde f(x).\le G<\infty.
  \end{equation}
\end{lem}
\begin{proof}
  For $x\in\cvxs$ define the affine function:
  \begin{equation}
    \label{eq:ext_convex3}
    a_x(y) = f(x)+\langle\partial f(x),y-x\rangle;
  \end{equation}
  then we set
  \begin{equation}
    \label{eq:ext_convex4}
    \tilde f(y) = \sup_{x\in\cvxs}a_x(y);
  \end{equation}
  then $\tilde f$ is convex being the pointwise $\sup$ of affine (and
  hence convex) functions;
  fix $y$ and take a maximizing sequence $\{x_n\}$ for the definition of
  $\tilde f(y)$; by compactness of $\cvxs$ and of the closed ball of radius
  $G$ in $\real^N$ we can find $z_y\in\cvxs$ and $v_y\in\real^N$ with $\hilbnrm
  v_y.\le G$ such that:
  \begin{equation}
    \label{eq:ext_convex5}
    \tilde f(y) = f(z_y)+\langle v_y,y-x\rangle.
  \end{equation}
  Now fix $y\in\cvxs$; evaluating the $\sup$ at $y$ gives $\tilde
  f(y)\ge f(y)$; on the other hand, for any other $z\in\cvxs$ the very
  definition of convexity and subgradient imply that $a_z(y)\le f(y)$
  and hence $\tilde f(y) = f(y)$. Finally, a bounded choice of the
  subgradient $\partial\tilde f(y)$ is obtained by choosing $v_y$.
\end{proof}
\subsection{Weakly enforcing constraints via penalization} We show how
to construct a penalization term $\psi_\cvxs$ to constrain the
membership of the parameters to $\cvxs$.
\begin{defn}[Support vectors]
  Let $\cvxs$ be a convex set; then for each $x\in\partial\cvxs$ let
  $S(x)$ denote the set of unit vectors such that
  \begin{equation}
    \sup_{y\in\cvxs}\langle v, y-x\rangle\le0;
  \end{equation}
  from convex analysis we know that $S(x)$ is nonempty and its elements
  are the \emph{support vectors} of $\cvxs$ at $x$.
\end{defn}
\begin{lem}[Penalization function]\label{lem:gauge}
  Let $\cvxs$ be a compact convex set; for $x\in\partial\cvxs$ and
  $v\in S(x)\cup \{0\}$ define the affine function:
  \begin{equation}
    \label{eq:gauge0}
    a_{x,v}(y) = \langle v, y -x\rangle;
  \end{equation}
  then we define the gauge:
  \begin{equation}
    \label{eq:gauge1}
    \psi_\cvxs(y)=\sup_{x\in\partial\cvxs, v\in S(x)\cup \{0\}}a_{x,v}(y).
  \end{equation}
  Then:
  \begin{itemize}
  \item $\psi_\cvxs$ is convex;
  \item there is a choice of the subgradient such that $\hilbnrm
    \partial\psi_\cvxs.\le 1$;
  \item $\psi_\cvxs=0$ on $\cvxs$;
  \item for $x\not\in\cvxs$ we have
    \begin{equation}
      \label{eq:gauge2}
      \psi_\cvxs(x)\ge\hilbnrm x-\pi_\cvxs(x)..
    \end{equation}
  \end{itemize}
\end{lem}
\begin{proof}
  The function $\psi_\cvxs$ is the pointwise $\sup$ of a family of
  affine functions, hence is convex. As the zero affine function is in
  the family, $\psi_\cvxs\ge0$ everywhere. But for each $y\in\cvxs$,
  from the definition of supporting vector, we also have
  $a_{x,v}(y)\le0$ so $\psi_\cvxs$ vanishes on $\cvxs$.
  For any $x\in\real^N$ a compactness argument gives an
  $x_y\in\partial\cvxs$
  and $v_y\in S(x_y)\cup\{0\}$ (hence a $v_y$ of norm
  at most $1$) and  such that:
  \begin{equation}
    \label{eq:gauge3}
    \psi_\cvxs(y)=a_{x_y,v_y}(y).
  \end{equation}
  From the following equations we see that we can use $y\mapsto v_y$
  as a subgradient at $y$:
  \begin{equation}
    \label{eq:gauge4}
    \begin{split}
      \psi_\cvxs(z) +\langle v_z, y-z\rangle &= \langle v_z,
      z-x_z\rangle + \langle v_z, y - z\rangle \\
      &=\langle v_z,y-x_z\rangle\\
      &\le \psi_\cvxs(y).
    \end{split}
  \end{equation}
  Let $x\not\in\cvxs$; then $\pi_\cvxs(x)\in\partial\cvxs$ and let $v$
  be the unit vector in the direction of $x-\pi_\cvxs(x)$; then by the
  minimizing properties of $\pi_\cvxs(x)$ for any $y\in\cvxs$ we have
  \begin{equation}
    \label{eq:gauge5}
      \langle v,y-x\rangle\le0,
    \end{equation}
    which implies $v\in S(\pi_{\cvxs}(x))$; but then we get~(\ref{eq:gauge2}).
  \end{proof}
  \begin{lem}[Constraint via penalization]
    \label{lem:conspen}
    Let $f:\real^N\to\real$ be (continuous) convex with $\hilbnrm \partial f.\le G$
    (for some choice of the subgradient); let $\cvxs\subset\real^N$ be
    compact convex and $x\optdec\in\cvxs$ be a minimizer of the
    restriction of $f$ to $\cvxs$. Then $f$ restricted on $\cvxs$
    and $f+2G\psi_\cvxs$ (on the whole $\real^N$) have
    the same minimizer; if for some $x\in\real^N$
    we have:
    \begin{equation}
      \label{eq:conspen0}
      f(x) + 2G\psi_\cvxs(x)-f(x\optdec)\le\varepsilon;
    \end{equation}
    then
    \begin{equation}
      \label{eq:conspen1}
      \hilbnrm x-\pi_\cvxs(x).\le\frac{\varepsilon}{G}.
    \end{equation}
  \end{lem}
  \begin{proof}
    To show that $f$ (restricted on $\cvxs$) and  $f+2G\psi_\cvxs$ (on the whole $\real^N$) have
    the same minimizer we argue by contradiction, assuming for some
    $u\not\in \cvxs$ we have
    \begin{equation}
      \label{eq:conspenadd1}
      f(\pi_\cvxs(u)) > f(u) + 2G\psi_\cvxs(u);
    \end{equation}
    then as $u\ne \pi_\cvxs(u)$ we get the contradiction:
    \begin{equation}
      \label{eq:conspenadd2}
      G\hilbnrm u-\pi_\cvxs(u). > 2G\hilbnrm u-\pi_\cvxs(u)..
    \end{equation}
    The proof of~\eqref{eq:conspen1} is immediate from Lemma~\ref{lem:gauge}:
    \begin{equation}
      \label{eq:conspen2}
      \begin{split}
        G\hilbnrm x-\pi_\cvxs(x). &\le f(x) +
        2G\psi_{\cvxs}(x)-f(\pi_\cvxs(x))\\
        &\le f(x) + 2G\psi_{\cvxs}(x)-f(x\optdec)\\
        &\le\varepsilon.
      \end{split}
    \end{equation}
  \end{proof}
  \begin{lem}[Approximate optimization]
    \label{lem:opt_and_approx}
    Let $f,\tilde f$ be real-valued functions on $\real^N$ with
    $|f-\tilde f|\le\varepsilon$; let $\cvxs\subset\real^N$ (not
    necessarily convex) and assume that $x\optdec$ is a minimizer of
    $f$ on $\cvxs$ and $\tilde x\optdec$ is a minimizer of $\tilde f$
    on $\cvxs$. Then if $x$ is a good candidate minimizer for $\tilde
    f$ it is so also for $f$:
    \begin{equation}
      \label{eq:opt_and_approx0}
      f(x)-f(x\optdec)\le 2\varepsilon + \tilde f(x) - \tilde f(\tilde x\optdec).
    \end{equation}
  \end{lem}
  \begin{proof}
    We apply two times the uniform closeness of $\tilde f$ and $f$ and
    the definition of minimizers:
    \begin{equation}
      \label{eq:opt_and_approx1}
      \begin{split}
        f(x)-f(x\optdec)&\le 2\varepsilon+\tilde f(x) - \tilde
        f(x\optdec)\\
        &\le 2\varepsilon+\tilde f(x) - \tilde f(\tilde x\optdec).
      \end{split}
    \end{equation}
  \end{proof}
\subsection{Smoothing} As $\partial f$ is not in general Lipschitz we resort to a
perturbation argument to regularize $f$ while staying close to
$f$. This technique is standard in real analysis, for a machine
learning reference see \cite[Sec.~2.3.2, Algorithm~4]{hazan-online-convex}.
  \begin{defn}[Mollifications]
    \label{defn:mollifications}
    Let
    \begin{equation}
      \label{eq:mollifications0}
      \eta_\varepsilon =
      \frac{1}{\varepsilon^N{{\rm Vol}(B(0,1))}}\chi_{B(0,1)},
    \end{equation}
    so that $\eta_\varepsilon$ is a probability distribution with
    mass absolutely continuous with respect to the Lebesgue measure.
    Moreover, $\eta_\varepsilon$ is a function of bounded variation
    and Stokes' Theorem shows that its gradient is:
    \begin{equation}
      \label{eq:mollifications1}
      D\eta_\varepsilon = -\frac{1}{\varepsilon^N
        {{\rm Vol}(B(0,1))}}\chi_{B(0,1)}\vec{S}(0,\varepsilon),
    \end{equation}
    where $\vec{S}(0,\varepsilon)$ is the signed measure on the
    boundary $\partial B(0,\varepsilon)$ where the positive direction
    is that of the outward normal. Comparing the surface area of $
    \partial B(0,\varepsilon)$ with the volume of $B(0,\varepsilon)$
    we obtain the total mass of $D\eta_\varepsilon$:
    \begin{equation}
      \label{eq:mollifications2}
      \hilbnrm D\eta_\varepsilon.(\real^N) = \frac{N}{\varepsilon}.
    \end{equation}
    Given a function $f:\real^N\to\real^M$ we can define the smoothing
    $f_\varepsilon$ as the expectation:
    \begin{equation}
      \label{eq:mollifications3}
      f_\varepsilon(x) = E_{\eta_\varepsilon}f(x-\cdot) =
      \int_{\real^N} f(x-v)\eta_\varepsilon(v)\,dv.
    \end{equation}
    If $f$ is convex then $f_\varepsilon$ is convex as we take an
    expectation of a family of convex functions. Similarly, if
    $\hilbnrm f.$ is bounded by a constant $G$ so is $\hilbnrm
    f_\varepsilon.$.
    For the gradient $\partial f_\varepsilon$ we have the formulas:
    \begin{equation}
      \label{eq:mollifications4}
      \partial f_\varepsilon(x) = \int \partial
      f(x-v)\eta_\varepsilon(v)\,dv =
      -\int f(x-v)\,dD\eta_\varepsilon(v),
    \end{equation}
    where the second integral is with respect to the measure
    $D\eta_\varepsilon$. Using~(\ref{eq:mollifications4}) we see that
    if $f$ is $G$-Lipschitz then $\partial f_\varepsilon$ is
    $GN/\varepsilon$-Lipschitz:
    \begin{equation}
      \label{eq:mollifications5}
      \partial f_\varepsilon(x) - \partial f_\varepsilon(y) =
      -\int (f(x-v)-f(y-v))\,dD\eta_\varepsilon(v),
    \end{equation}
    from which
    \begin{equation}
      \label{eq:mollifications6}
      \hilbnrm       \partial f_\varepsilon(x) - \partial
      f_\varepsilon(y). \le G\hilbnrm x-y.\hilbnrm D\eta_\varepsilon.(\real^N).
    \end{equation}
    Finally we also have the bound:
    \begin{equation}
      \label{eq:mollifications7}
      |f_\varepsilon(x)-f(x)| \le\int_{\real^N}G\hilbnrm
      v.\eta_\varepsilon(v)\,dv \le G\varepsilon.
    \end{equation}
  \end{defn}
\subsection{Analysis of SSGD} We now analyze the convergence of
SSGD. The main ideas of the proof are the use of $\psi_\cvxs$,
using martingale bounds combined with the analysis of gradient descent
on the expected smoothed objective $g$, and a case by case analysis
(Step 5: this is the place of the argument that should be improved to
speed up convergence).
\begin{thm}[Analysis of SSGD]
  \label{thm:anassgd}
  Assume a uniform bound $G$ on the norms of $\partial f_{t+1/2}$,
  $\partial f$; let $u\optdec$ be a minimizer of $f$ on $\cvxs$; in
  the asymptotic regime where $\varepsilon\searrow 0$ one has that if
  \begin{equation}
    \label{eq:anassgd_stat0}
T\ge \left\lceil
   \frac{3G\diam\cvxs}{2\varepsilon^3}+1
 \right\rceil^{2} -1 = O(\varepsilon^{-6}),
\end{equation}
with probability $\ge 1-2\varepsilon$ the algorithm SSGD
returns a final point $u\findec$ which minimizes $f$ up to an error
$O(\varepsilon\log(1/\varepsilon))$:
\begin{equation}
  \label{eq:anassgd_stat1}
  f(u\findec)-f(u\optdec)\lessapprox
  128G^2\varepsilon\log(1/\varepsilon).
\end{equation}
On the other hand, in the case in which $\partial f$ is $L$-Lipschitz,
one can achieve~\eqref{eq:anassgd_stat1} for
  \begin{equation}
    \label{eq:anassgd_stat2}
T\ge \left\lceil
   \frac{3G\diam\cvxs}{2\varepsilon^2}+1
 \right\rceil^{2} -1 = O(\varepsilon^{-4}).
\end{equation}
\end{thm}
\begin{proof}
  \MakeStep{Step 1: A gradient descent bound.}
  The function that SSGD is effectively trying to minimize is the
  smoothing:
  \begin{equation}
    \label{eq:anassgd1}
    g = E_{\eta_{\varepsilon\smdec}}[f+2G\psi_\cvxs]
  \end{equation}
  whose subgradient $\partial g$ satisfies the bound:
  \begin{equation}
    \label{eq:anassgd2}
    \hilbnrm\partial g. \le 3G
  \end{equation}
  and is $\frac{3GN}{\varepsilon\smdec}$-Lipschitz
  by~\eqref{eq:mollifications6};
  from now on we
  will use several times the bound~\eqref{eq:mollifications7}
  which implies that $|g-(f+2G\psi_\cvxs)|\le 3G\varepsilon\smdec$.
  Moreover at time $t+1/2$ we are using the function:
  \begin{equation}
    \label{eq:anassgd3}
    g_{t+1/2} = f_{t+1/2}(\cdot - v) + 2G\psi_\cvxs(\cdot -v),
  \end{equation}
  to decide the gradient step where $v$ is the point sampled from
  $\eta_{\varepsilon\smdec}$. Using the argument in
  Lemma~\ref{lem:ppgdlem} we get
  \begin{equation}
    \label{eq:anassgd4}
    g(u_{t+1})-g(u_t) \le
    \frac{27G^3N}{\varepsilon\smdec}\varepsilon_t^2 -
    \varepsilon_t\langle\partial g(u_t), \partial g_{t+1/2}(u_t)\rangle;
  \end{equation}
  then for $m\ge 1$ we get:
  \begin{equation}
    \label{eq:anassgd5}
    g(u_{t+m})-g(u_t)
    \le\frac{27G^3N}{\varepsilon\smdec}\sum_{s=t}^{t+m-1}\varepsilon_s^2
    -\sum_{s=t}^{t+m-1}\varepsilon_s\langle\partial g(u_s),
    \partial g_{s+1/2}(u_s)\rangle.
  \end{equation}
  \MakeStep{Step 2: A martingale bound. }
  Let $\filtration_t$ denote the filtration at time $t$ (which can be
  integer or integer plus $1/2$); SSGD gives rise to a random variable
  \begin{equation}
    \label{eq:anassgd6}
    X_T = -\sum_{s=1}^T\varepsilon_s[\langle\partial g(u_s),
    \partial g_{s+1/2}(u_s)\rangle - \hilbnrm\partial g(u_s).^2];
  \end{equation}
  then for $t\le T$ define
  \begin{equation}
    \label{eq:anassgd7}
    X_t = E[X_T|\filtration_{t+1}],
  \end{equation}
  so that $\{X_t\}_t$ defines a martingale. As $g_{s+1/2}$ is sampled
  independent of the filtration $\filtration_s$ we find:
  \begin{equation}
    \label{eq:anassgd8}
    X_t = -\sum_{s=1}^t\varepsilon_s[\langle\partial g(u_s),
    \partial g_{s+1/2}(u_s)\rangle - \hilbnrm\partial g(u_s).^2];
  \end{equation}
  in particular we have the bound:
  \begin{equation}
    \label{eq:anassgd9}
    | X_t - X_{t-1}| \le 18G^2\varepsilon_t.
  \end{equation}
  We can then use Hoeffding's inequality in the form
  explained in
  \cite[equation~2.17]{hoeffding-concentration}
  (this gives a slightly better bound and it uses one of Doob's maximal
  inequalities) to obtain:
  \begin{equation}
    \label{eq:anassgd10}
    P(\max_t | X_t | \ge \varepsilon\probdec) \le 2\exp\left(
-\frac{\varepsilon^2\probdec}{648G^4\sum_{t=1}^T\varepsilon_t^2}
      \right).
    \end{equation}
    \MakeStep{Step 3: Combining the martingale and the gradient
      descent bounds.}
    We now let $u^g\optdec\in\real^N$ be an optimal point for $g$; by
    Lemma~\ref{lem:opt_and_approx} a minimizer $u\optdec\in\cvxs$
    of $f + 2G\psi_\cvxs$ (and of $f$ restricted on $\cvxs$ by
    Lemma~\ref{lem:conspen}) is also almost a minimizer of the
    smoothing $g$:
    \begin{equation}
      \label{eq:anassgd11}
      g(u\optdec)\le 6G\varepsilon\smdec + g(u^g\optdec).
    \end{equation}
    Let $\Omega$ denote the set of events where one has:
    \begin{equation}
      \label{eq:anassgd12}
      \max_t |X_t| \le\varepsilon\probdec. 
    \end{equation}
    Combining~\eqref{eq:anassgd5} with~\eqref{eq:anassgd11}--\eqref{eq:anassgd12} we get:
    \begin{equation}
      \label{eq:anassgd13}
      g(u_{T+1})-g(u^g\optdec)
    \le\frac{27G^3N}{\varepsilon\smdec}\sum_{s=1}^{T}\varepsilon_s^2
    -\sum_{s=1}^{T}\varepsilon_s \hilbnrm\partial g(u_s).^2 +
    g(u_1)-g(u\optdec) + 6G\varepsilon\smdec + \varepsilon\probdec.
  \end{equation}
  \MakeStep{Step 4: Bounding the distance of the $u_t$'s from the
    convex set.}
  We will now prove a bound on the distance of any $u_t$ from the
  set $\cvxs$:
  \begin{equation}
    \label{eq:anassgd13bis}
    \begin{split}
      G\hilbnrm u_t - \pi_{\cvxs}(u_t). &\le (f+2G\psi_\cvxs)(u_t) -
      f(\pi_\cvxs(u_t)) \le 6G\varepsilon\smdec + g(u_t)-g(\pi_\cvxs(u_t)) \\
      &\le 6G\varepsilon\smdec + g(u_t)-g(\pi_\cvxs(u_t)) +
      g(\pi_\cvxs(u_t)) - g(u^g\optdec) \\
      &\le 12G\varepsilon\smdec + g(u_t) - g(u\optdec)\\
      &\le\frac{27G^3N}{\varepsilon\smdec}\sum_{s=1}^{T}\varepsilon_s^2
      -\sum_{s=1}^{T}\varepsilon_s \hilbnrm\partial g(u_s).^2 +
      \varepsilon\probdec + 12G\varepsilon\smdec + g(u_1)-g(u\optdec)\\
      &\le
      \frac{27G^3N}{\varepsilon\smdec}\sum_{s=1}^{T}\varepsilon_s^2
      +\varepsilon\probdec + 12G\varepsilon\smdec +
      3G\hilbnrm u_1-u\optdec.;
    \end{split}
  \end{equation}
  from which we get:
  \begin{equation}
    \label{eq:anassgd13tris}
    \textrm{dist}(u_t,\cvxs) \le\underbrace{\frac{27G^2N}{\varepsilon\smdec}\sum_{s=1}^{T}\varepsilon_s^2
      +\frac{\varepsilon\probdec + 12G\varepsilon\smdec}{G} +
      3\diam\cvxs}_{=d_\varepsilon}.
  \end{equation}
  \MakeStep{Step 5: Bounding the algorithm in different cases. }
  We now need to make an analysis in different cases; in case
  \textbf{C1} we have:
    \begin{equation}
      \label{eq:anassgd14}
      \frac{27G^3N}{\varepsilon\smdec}\sum_{s=1}^{T}\varepsilon_s^2
    -\sum_{s=1}^{T}\varepsilon_s \hilbnrm\partial g(u_s).^2 +
    g(u_1)-g(u\optdec) + 6G\varepsilon\smdec + \varepsilon\probdec \le
    \varepsilon\errdec;
  \end{equation}
  this implies
  \begin{equation}
    \label{eq:anassgd15}
    g(u_{T+1})-g(u^g\optdec)\le\varepsilon\errdec.
  \end{equation}
  In case \textbf{C1} is violated we have:
  \begin{equation}
    \label{eq:anassgd16}
    \sum_{s=1}^{T}\varepsilon_s \hilbnrm\partial g(u_s).^2
    \le\frac{27G^3N}{\varepsilon\smdec}\sum_{s=1}^{T}\varepsilon_s^2
    +
    g(u_1)-g(u\optdec) + 6G\varepsilon\smdec + \varepsilon\probdec
    -\varepsilon\errdec;
  \end{equation}
  in this case there must be values of $t$ for which one has:
  \begin{equation}
    \label{eq:anassgd17}
    \hilbnrm\partial g(u_t).\le\underbrace{\left(
      \frac{\frac{27G^3N}{\varepsilon\smdec}\sum_{s=1}^{T}\varepsilon_s^2
    +
    g(u_1)-g(u\optdec) + 6G\varepsilon\smdec + \varepsilon\probdec
    -\varepsilon\errdec}
  {\sum_{s=1}^T\varepsilon_s}
      \right)^{1/2}}_{\textrm{BG1}};
    \end{equation}
    let $t^*$ be the maximal such $t$ satisfying~\eqref{eq:anassgd17}.
    Combining~\eqref{eq:anassgd5} with the fact that on $\Omega$
    \eqref{eq:anassgd12} holds we get:
    \begin{equation}
      \label{eq:anassgd18}
      g(u_{T+1})-g(u_{t^*}) \le
      \frac{27G^3N}{\varepsilon\smdec}\sum_{s=t^*}^{T}\varepsilon_s^2
      +2\varepsilon\probdec-
      \sum_{s=t*}^{T}\varepsilon_s \hilbnrm\partial g(u_s).^2.
    \end{equation}
    In case \textbf{C2} we have:
    \begin{equation}
      \label{eq:anassgd19}
      \frac{27G^3N}{\varepsilon\smdec}\sum_{s=t^*}^{T}\varepsilon_s^2
      -      \sum_{s=t*}^{T}\varepsilon_s \hilbnrm\partial
      g(u_s).^2\le 0;
    \end{equation}
    in this case we combine \eqref{eq:anassgd11},
    \eqref{eq:anassgd13tris},
    \eqref{eq:anassgd17} and \eqref{eq:anassgd18} to obtain:
    \begin{equation}
      \label{eq:anassgd20}
      \begin{split}
        g(u_{T+1})-g(u^g\optdec) &\le g(u_{t^*})-g(u\optdec)
        + g(u\optdec)-g(u^g\optdec) + 2\varepsilon\probdec \\
        &\le \textrm{BG1}\times(d_\varepsilon+\diam\cvxs) + 2\varepsilon\probdec
      + 6G\varepsilon\smdec.
      \end{split}
    \end{equation}
    If case \textbf{C2} fails we are in case \textbf{C3} in which we
    see that the set of those $s\ge t^*$ such that:
    \begin{equation}
      \label{eq:anassgd21}
      \hilbnrm\partial g(u_s).^2 \le \frac{27G^3N}{\varepsilon\smdec}
      \frac{\sum_{t\ge s}\varepsilon_t^2}{\sum_{t\ge s}\varepsilon_t}
    \end{equation}
    is not empty. Let $t^{**}$ be a maximal $s\in\{t^*,\cdots,T\}$
    satisfying~\eqref{eq:anassgd21}. Let:
    \begin{equation}
      \label{eq:anassgd22}
      \textrm{BG2} = \left(
        \frac{27G^3N}{\varepsilon\smdec}
      \frac{\sum_{s\ge t^{**}}\varepsilon_s^2}{\sum_{s\ge t^{**}}\varepsilon_s}
        \right)^{1/2}.
    \end{equation}
    Then arguing as we
    obtained~\eqref{eq:anassgd18} we get:
    \begin{equation}
      \label{eq:anassgd23}
     g(u_{T+1})-g(u_{t^{**}+1}) \le
      \frac{27G^3N}{\varepsilon\smdec}\sum_{s={t^{**}+1}}^{T}\varepsilon_s^2
      +2\varepsilon\probdec-
      \sum_{s={t^{**}+1}}^{T}\varepsilon_s \hilbnrm\partial g(u_s).^2.
    \end{equation}
    By maximality of $t^{**}$ we have
    \begin{equation}
      \label{eq:anassgd24}
      \frac{27G^3N}{\varepsilon\smdec}\sum_{s={t^{**}+1}}^{T}\varepsilon_s^2
      -
      \sum_{s={t^{**}+1}}^{T}\varepsilon_s \hilbnrm\partial g(u_s).^2
      \le0;
    \end{equation}
    now combining \eqref{eq:anassgd11}, \eqref{eq:anassgd13bis},
    \eqref{eq:anassgd21} and \eqref{eq:anassgd22} we get:
    \begin{equation}
      \label{eq:anassgd25}
      \begin{split}
        g(u_{T+1}) - g(u^g\optdec) &\le g(u_{T+1}) - g(u_{t^{**}+1})
        + g(u_{t^{**}+1}) - g(u_{t^{**}}) \\
        &\mskip 6mu+ g(u_{t^{**}})-g(u\optdec) + g(u\optdec)
        -g(u^g\optdec)\\
        &\le 2\varepsilon\probdec + 3G\varepsilon_{t^{**}} +
        \textrm{BG2}\times (d_\varepsilon+\diam\cvxs) \\
        &\mskip 6mu +6G\varepsilon\smdec.
      \end{split}
    \end{equation}
    \MakeStep{Step 6: Choice of the learning rate sequence and asymptotic bounds.}
    We start by requiring $\varepsilon\probdec=\varepsilon\errdec$ and
    setting
    \begin{equation}
      \label{eq:anassgd26}
      \varepsilon_s = \varepsilon\smdec\frac{1}{\sqrt{s}};
    \end{equation}
    from the bounds
    \begin{equation}
      \label{eq:anassgd27}
      \frac{1}{\sqrt{t}}\chi_{[t_1,t_2+1]}(t) \le
      \sum_{s=t_1}^{t_2}\frac{1}{\sqrt{s}}\chi_{[s,s+1]}(t)
      \le \frac{1}{\sqrt{t-1}}\chi_{[\max(t_1,2),t_2+1]} + \chi_{t_1=1}\chi_{[1,2]};
    \end{equation}
   we derive
    \begin{equation}
      \label{eq:anassgd28}
      \begin{split}
        \varepsilon\smdec^2(\log(t_2+1)-\log(t_1)) &\le
        \sum_{s=t_1}^{t_2}\varepsilon_s^2 \\
        &\le \varepsilon\smdec^2(\chi_{t_1=1}+\log(t_2)-\log(\max(t_1,2)-1)),
      \end{split}
    \end{equation}
    and
    \begin{equation}
      \label{eq:anassgd29}
      \begin{split}
      2\varepsilon\smdec(\sqrt{t_2+1}-\sqrt{t_1}) &\le
      \sum_{s=t_1}^{t_2}\varepsilon_s\\ &\le
      2\varepsilon\smdec(\chi_{t_1=1}
      +\sqrt{t_2}-\sqrt{\max(t_1,2)-1}).
    \end{split}
  \end{equation}
  Substitution of these bounds in~(\ref{eq:anassgd17}) yields:
  \begin{equation}
    \label{eq:anassgd30}
    \textrm{BG1}\le\left(
      \frac{27G^3N(1+\log T)+6G}{2(\sqrt{T+1}-1)} +
      \frac{3G\diam\cvxs}
      {2\varepsilon\smdec(\sqrt{T+1}-1)}
      \right)^{1/2};
    \end{equation}
    making the choice of the $T(\varepsilon\smdec)$ as:
    \begin{equation}
      \label{eq:anassgd31}
      T = \left\lceil
   \frac{3G\diam\cvxs}{2\varepsilon\smdec^3}+1
        \right\rceil^{2} -1 = O(\varepsilon\smdec^{-6})
      \end{equation}
      we get that in the asymptotic case $\varepsilon\smdec\searrow0$ one
      has:      
      \begin{equation}
        \label{eq:anassgd32}
        \textrm{BG1}\lessapprox\varepsilon\smdec.
      \end{equation}
      Using the bounds in~(\ref{eq:anassgd22}) we get:
      \begin{equation}
        \label{eq:anassgd33}
        \begin{split}
          \textrm{BG2} &\le \left(
            \frac{27G^3N(\log T +
              \chi_{t^{**}=1}-\log(\max(t^{**},2)-1))}
             {2(\sqrt{T+1}-\sqrt{t^{**}})}
           \right)^{1/2} \\
           &\le\left(\frac{27G^3N}{2}\frac{\log T - \log(T-1)}
        {\sqrt{T+1}-\sqrt{T}}
      \right)^{1/2}\\
      &\le\left(
        27G^3N\frac{\sqrt{T+1}}{T-1}
        \right)^{1/2};
        \end{split}
      \end{equation}
      from which we get the asymptotic bound
      \begin{equation}
        \label{eq:anassgd34}
        \textrm{BG2}\lessapprox\left(
          \frac{18G^2N}{\diam\cvxs}
          \right)^{1/2}\varepsilon\smdec^{3/2}.
        \end{equation}
        We now turn to~(\ref{eq:anassgd10}), making the choice
        \begin{equation}
          \label{eq:anassgd35}
          \varepsilon\probdec = 64G^2\varepsilon\smdec\log(1/\varepsilon\smdec)
        \end{equation}
        we get:
        \begin{equation}
          \label{eq:anassgd36}
          \begin{split}
          P(\max_t | X_t | \ge
          64G^2\varepsilon\smdec\log(1/\varepsilon\smdec))
          &\le 2\\
          &\mskip 6mu\times\exp\left(
             -\frac{\varepsilon\smdec^2(\log(1/\varepsilon\smdec))^2
               (64G)^2}
             {648G^4\varepsilon\smdec^2(\log(
               \left\lceil
    \frac{3G\diam\cvxs}{2\varepsilon\smdec^3}+1
  \right\rceil^{2} -1)+1)}\right)\\
&\lessapprox 2\exp(-\log(1/\varepsilon\smdec))=2\varepsilon\smdec.
          \end{split}
        \end{equation}
        Now in the asymptotic regime when
        $\varepsilon\smdec\searrow0$, the quantity
        $\varepsilon\probdec$ dominates
        in the bounds \eqref{eq:anassgd15} (case \textbf{C1}),
        \eqref{eq:anassgd20} (case \textbf{C2}),
        \eqref{eq:anassgd25} (case \textbf{C3}); i.e.~on $\Omega$ one
        has that
        \begin{equation}
          \label{eq:anassgd37}
          g(u_{T+1}) - g(u^g\optdec) \lessapprox 2\varepsilon\probdec
          = 128G^2\varepsilon\smdec\log(1/\varepsilon\smdec).
        \end{equation}
        We can now obtain an asymptotic bound for $u\findec$ being an
        approximate minimizer of $f$ on $\cvxs$:
        \begin{equation}
          \label{eq:anassgd38}
          \begin{split}
            f(u\findec)-f(u\optdec) &\le \underbrace{f(u\findec)
              - (f(u_{T+1})+2G\psi_\cvxs(u_{T+1}))}_{\text{$\le 0$ by
                \eqref{eq:gauge2}}}\\
            &\mskip 6mu + (f(u_{T+1})+2G\psi_\cvxs(u_{T+1})) -
            f(u\optdec) \\
            &\le 2\varepsilon\smdec + g(u_{T+1})-g(u\optdec)\\
            &\le 2\varepsilon\smdec + g(u_{T+1})-g(u^g\optdec)
            +\underbrace{g(u^g\optdec)-g(u\optdec)}_{\le 0} \\
            &\lessapprox 2\varepsilon\probdec.
          \end{split}
        \end{equation}
        Now \eqref{eq:anassgd_stat1} follows by setting
        $\varepsilon\smdec=\varepsilon$.
        \MakeStep{Step 7: Proof of~\eqref{eq:anassgd_stat2}.}
        In this case a minor variation is required, we leave most
        details to the reader. The point is that no smoothing is
        required and thus one can directly work with $f+2G\psi_\cvxs$
        and then set $\varepsilon_s=\frac{1}{\sqrt{s}}$.
      \end{proof}
\subsection{Analysis of RASPGD} In this section we analyze RASPGD. The
argument is close to that of RAPGD, the main difference is the use of
a concentration inequality argument.
\begin{thm}[Analysis of RASPGD]
  \label{thm:raspgd}
    Assume a uniform bound $G$ on the norms of $\partial f_{t+1/2}$,
  $\partial f$; let $u\optdec$ be a minimizer of $f$ on $\cvxs$; in
  the asymptotic regime where $\varepsilon\searrow 0$ one has that if
  \begin{equation}
    \label{eq:raspgd_stat0}
          T\ge\lceil \frac{2G\diam\cvxs+G^2}{4\varepsilon}\rceil^2=O(\varepsilon^{-2})
\end{equation}
with probability at least
\begin{equation}
  \label{eq:raspgd_stat1}
1-  \frac{32\varepsilon^2}{(2G\diam\cvxs+G^2)^2}      = 1-O(\varepsilon^2)
\end{equation}
the algorithm RASPGD
returns a final point $u\findec$ which minimizes $f$ up to an error
$O(\varepsilon\log(1/\varepsilon))$:
\begin{equation}
  \label{eq:raspgd_stat2}
  f(u\findec)-f(u\optdec)\lessapprox
  2\varepsilon\log(1/\varepsilon).
\end{equation}
\end{thm}
\begin{proof}
  \MakeStep{Step 1: Generalize the bounds in Theorem~\ref{thm:rapgd}.}
  Equation~\eqref{eq:rapgd2} has been established at the level of each
  iteration, so in RASPGD we get:
  \begin{equation}
    \label{eq:raspgd0}
      \varepsilon_t (f_{t+1/2}(u_t)-f_{t+1/2}(u\optdec)) \le
  \frac{\hilbnrm u_t-u\optdec.^2 - \hilbnrm u_{t+1}-u\optdec.^2
    +G^2\varepsilon_t^2}{2};
\end{equation}
thus we can also generalize~\eqref{eq:rapgd3}:
\begin{equation}
  \label{eq:raspgd1}
  \sum_{t=1}^T   \varepsilon_t (f_{t+1/2}(u_t)-f_{t+1/2}(u\optdec)) \le
  \frac{\hilbnrm u_1-u\optdec.^2+G^2\sum_{t=1}^T\varepsilon_t^2
          }{2}.
        \end{equation}
   \MakeStep{Step 2: A martingale bound. }
  Let $\filtration_t$ denote the filtration at time $t$ (which can be
  integer or integer plus $1/2$); RASPGD gives rise to a random variable
  \begin{equation}
    \label{eq:raspgd2}
    X_T = \sum_{s=1}^T\varepsilon_s[
    (f_{s+1/2}(u_s)-f_{s+1/2}(u\optdec)) - (f(u_s)-f(u\optdec))
    ];
  \end{equation}
  then for $t\le T$ define
  \begin{equation}
    \label{eq:raspgd3}
    X_t = E[X_T|\filtration_{t+1}],
  \end{equation}
  so that $\{X_t\}_t$ defines a martingale.
  As $f_{s+1/2}$ is sampled
  independent of the filtration $\filtration_s$ we find:
  \begin{equation}
    \label{eq:raspgd4}
    X_t = \sum_{s=1}^t\varepsilon_s[
    (f_{s+1/2}(u_s)-f_{s+1/2}(u\optdec)) - (f(u_s)-f(u\optdec))
    ];
  \end{equation}
  in particular we have the bound:
  \begin{equation}
    \label{eq:raspgd5}
    | X_t - X_{t-1}| \le 2G\diam\cvxs\varepsilon_t.
  \end{equation}
    We can then use Hoeffding's inequality
  \cite[equation~2.17]{hoeffding-concentration}
  (this gives a slightly better bound and it uses one of Doob's maximal
  inequalities) to obtain:
  \begin{equation}
    \label{eq:raspgd6}
    P(\max_t | X_t | \ge \varepsilon\probdec) \le 2\exp\left(
-\frac{\varepsilon^2\probdec}{4G^2(\diam\cvxs)^2\sum_{t=1}^T\varepsilon_t^2}
      \right).
    \end{equation}
    Let $\Omega$ denote the set of event where one has:
    \begin{equation}
      \label{eq:raspgd7}
      |X_T|\le\varepsilon\probdec.
    \end{equation}
    \MakeStep{Step 3: Bounding the Algorithm on $\Omega$.}
    Using~\eqref{eq:raspgd1} and the definitions of $X_T$ and $\Omega$
    we conclude that on $\Omega$:
    \begin{equation}
      \label{eq:raspgd8}
      \sum_{t=1}^T\varepsilon_t(f(u_t)-f(u\optdec)) \le
      \frac{(\diam\cvxs)^2+G^2\sum_{t=1}^T\varepsilon_t^2}
      {2}+\varepsilon\probdec;
    \end{equation}
    application of Jensen's inequality finally yields:
    \begin{equation}
      \label{eq:raspgd9}
      f(u\findec)-f(u\optdec)\le\frac{(\diam\cvxs)^2+G^2\sum_{t=1}^T\varepsilon_t^2
      + 2\varepsilon\probdec}
      {2\sum_{t=1}^T\varepsilon_t}.
    \end{equation}
    \MakeStep{Step 4: Choice of the sequence $\varepsilon_t$.} We set
    \begin{equation}
      \label{eq:raspgd10}
      \varepsilon_t = \frac{1}{\sqrt{t}}
    \end{equation}
    and
    \begin{equation}
      \label{eq:raspgd11}
      \varepsilon\probdec = 2G\diam\cvxs (\log T + 1);
    \end{equation}
    in analogy with \eqref{eq:anassgd27}--\eqref{eq:anassgd29}
    we get:
    \begin{equation}
      \label{eq:raspgd12}
      P(\Omega^c) \le \frac{2}{T}
    \end{equation}
    and
    \begin{equation}
      \label{eq:raspgd13}
      \begin{split}
        f(u\findec)-f(u\optdec) &\le \frac{(\diam\cvxs)^2 + G^2(\log T
          + 1) + 2G\diam\cvxs (\log T + 1)}{4(\sqrt{T}-1)} \\
        &\lessapprox\frac{2G\diam\cvxs + G^2}{4}\frac{\log T}{\sqrt{T}}.
      \end{split}
    \end{equation}
    If we choose
    \begin{equation}
      \label{eq:raspgd14}
      T\ge\lceil \frac{2G\diam\cvxs+G^2}{4\varepsilon}\rceil^2=O(\varepsilon^{-2})
    \end{equation}
    we then get
    \begin{equation}
      \label{eq:raspgd15}
      P(\Omega^c)\lessapprox\frac{32\varepsilon^2}{(2G\diam\cvxs+G^2)^2}      
    \end{equation}
    and
    \begin{equation}
      \label{eq:raspgd16}
      f(u\findec)-f(u\optdec)\lessapprox 2\varepsilon\log(1/\varepsilon).
    \end{equation}
  \end{proof}
\bibliographystyle{alpha}
\bibliography{plain_sgd}
\end{document}